\documentclass[twoside,11pt]{article}
\usepackage{journal}
\usepackage{natbib}
\usepackage{bbm,bm}
\usepackage{epsfig,ifthen}
\usepackage{latexsym}
\usepackage{amsfonts}
\usepackage{eepic}
\usepackage{amsopn}
\usepackage{amsmath,amssymb,rotating,graphicx,rotating,float}
\usepackage{accents}
\usepackage[dvips,backref,pageanchor=true,plainpages=false, 
pdfpagelabels,bookmarks,bookmarksnumbered,breaklinks,
pdfborder={0 0 0},  
]{hyperref}

\usepackage{cite}
\usepackage[mathscr]{eucal}
\usepackage{url} 

\usepackage{ulem}
\usepackage{color} 

\newcommand{\Px}{\mathcal{P}}

\newcommand{\Maj}{{\rm Majority}} 
\newcommand{\SC}{\mathcal{M}} 

\renewcommand{\dim}{d} 

\newcommand{\Log}{{\rm Log}}

\newcommand{\conf}{\delta}
\newcommand{\eps}{\varepsilon}

\newcommand{\X}{\mathcal X} 
\newcommand{\Y}{\mathcal Y} 
\newcommand{\genericalg}{\mathcal{A}} 
\newcommand{\alg}{\mathbb{A}} 
\renewcommand{\H}{\mathcal H} 

\newcommand{\target}{f^{\star}} 
\newcommand{\er}{{\rm er}} 

\newcommand{\ind}{\mathbbm{1}}

\newcommand{\C}{\mathbb{C}} 

\newcommand{\sign}{{\rm sign}} 
\renewcommand{\P}{\mathbb P} 
\newcommand{\nats}{\mathbb{N}} 

\newcommand{\E}{\mathbb{E}}

\newcommand{\hmaj}[1]{h_{#1}} 
\newcommand{\hmaji}{\hmaj{i}} 
\newcommand{\Data}{\mathbb{S}}

\newcommand{\ProcX}{\mathbb{X}}

\renewcommand{\a}{a} 
\renewcommand{\b}{b} 

\newsavebox{\savepar}
\newenvironment{bigboxit}{\begin{center}\begin{lrbox}{\savepar}
\begin{minipage}[h]{5.6in}
\normalfont
\begin{flushleft}}
{\end{flushleft}\end{minipage}\end{lrbox}\fbox{\usebox{\savepar}}
\end{center}}

\makeatletter
\newcommand{\vast}{\bBigg@{3}}
\newcommand{\Vast}{\bBigg@{4}}
\makeatother

\renewenvironment{proof}[1][]{\par\noindent{\bf Proof #1\ }}{\hfill\BlackBox\\[2mm]}

\jmlrheading{}{2015}{}{}{}{Steve Hanneke}

\ShortHeadings{The Sample Complexity of PAC Learning}{Steve Hanneke}
\firstpageno{1}

\begin{document}

\title{The Optimal Sample Complexity of PAC Learning}

\author{%
\name Steve Hanneke \email steve.hanneke@gmail.com
}


\maketitle

\begin{abstract}%
This work establishes a new upper bound on the 
number of samples sufficient for PAC learning in the realizable case.
The bound matches known lower bounds up to numerical constant factors.
This solves a long-standing open problem on the sample complexity of PAC learning.
The technique and analysis build on a recent breakthrough by Hans Simon.
\end{abstract}


\section{Introduction}
\label{sec:intro}

Probably approximately correct learning (or \emph{PAC} learning; \citealp*{valiant:84})
is a classic criterion for supervised learning, which has been
the focus of much research in the past three decades.  
The objective in PAC learning is to produce a classifier 
that, with probability at least $1-\conf$, has error rate
at most $\eps$.  To qualify as a PAC learning algorithm, it must 
satisfy this guarantee for all possible target concepts in a given family,
under all possible data distributions.
To achieve this objective, the learning algorithm
is supplied with a number $m$ of i.i.d.\! training samples (data points), 
along with the corresponding correct classifications.  One of the central
questions in the study of PAC learning is determining the minimum
number $\SC(\eps,\conf)$ of training samples necessary and sufficient 
such that there exists a PAC learning algorithm requiring at most 
$\SC(\eps,\conf)$ samples (for any given $\eps$ and $\conf$).
This quantity $\SC(\eps,\conf)$ is known as the \emph{sample complexity}.

Determining the sample complexity of PAC learning
is a long-standing open problem.  
There have been upper and lower bounds established
for decades, but they differ by a logarithmic factor.  It has been widely
believed that this logarithmic factor can be removed
for certain well-designed learning algorithms, and attempting
to prove this has been the subject of much effort.  
\citet*{simon:15} has very recently made an enormous leap forward 
toward resolving this issue.  
That work proposed an algorithm that classifies points based on 
a majority vote among classifiers trained on independent data sets.
Simon proves that this algorithm achieves a sample complexity
that reduces the logarithmic factor in the upper bound down to a
very slowly-growing function.
However, that work does not quite completely resolve the gap,
so that determining the optimal sample complexity remains open.

The present work resolves this problem by completely
eliminating the logarithmic factor.
The algorithm achieving this new bound is also based on a 
majority vote of classifiers.  However, unlike Simon's algorithm,
here the voting classifiers are trained on data subsets specified by 
a recursive algorithm, with substantial overlaps among the
data subsets the classifiers are trained on.

\section{Notation}
\label{sec:notation}

We begin by introducing some basic notation
essential to the discussion.
Fix a nonempty set $\X$, called the \emph{instance space};
we suppose $\X$ is equipped with a $\sigma$-algebra, defining
the measurable subsets of $\X$.
Also denote $\Y = \{-1,+1\}$, called the \emph{label space}.
A \emph{classifier} is any measurable function $h : \X \to \Y$.
Fix a nonempty set $\C$ of classifiers, called the \emph{concept space}.
To focus the discussion on nontrivial cases,\footnote{The 
sample complexities for $|\C| = 1$ and $|\C| = 2$ are already quite well understood
in the literature, the former having sample complexity $0$, and the latter having 
sample complexity either $1$ or $\Theta(\frac{1}{\eps}\ln\frac{1}{\conf})$
(depending on whether the two classifiers are exact complements or not).}  
we suppose $|\C| \geq 3$;
other than this, the results in this article will be valid for \emph{any} choice of $\C$.

In the learning problem, there is a probability measure $\Px$ over $\X$,
called the \emph{data distribution}, and a sequence $X_{1}(\Px),X_{2}(\Px),\ldots$ of 
independent $\Px$-distributed random variables, 
called the \emph{unlabeled data};
for $m \in \nats$, also define $\ProcX_{1:m}(\Px) = \{X_{1}(\Px),\ldots,X_{m}(\Px)\}$,
and for completeness denote $\ProcX_{1:0}(\Px) = \{\}$.
There is also a special element of $\C$, denoted $\target$,
called the \emph{target function}.  
For any sequence $S_{x} = \{x_{1},\ldots,x_{k}\}$ in $\X$,
denote by $(S_{x},\target(S_{x})) = \{ (x_{1},\target(x_{1})), \ldots, (x_{k},\target(x_{k})) \}$.
For any probability measure $P$ over $\X$, and any classifier $h$, 
denote by $\er_{P}(h;\target) = P( x : h(x) \neq \target(x) )$.
A \emph{learning algorithm} $\genericalg$ is a map,\footnote{We also admit randomized algorithms $\genericalg$,
where the ``internal randomness'' of $\genericalg$ is assumed to be independent of the data.
Formally, there is a random variable $R$ independent of $\{X_{i}(P)\}_{i,P}$ such that 
the value $\genericalg(S)$ is determined by the input data $S$ and the value of $R$.}
mapping any sequence $\{(x_{1},y_{1}),\ldots,(x_{m},y_{m})\}$ in $\X \times \Y$ (called a \emph{data set}),
of any length $m \in \nats \cup \{0\}$, to a classifier $h : \X \to \Y$
(not necessarily in $\C$).

\begin{definition}
\label{def:sc}
For any $\eps,\conf \in (0,1)$, the 
\emph{sample complexity of $(\eps,\conf)$-PAC learning}, 
denoted $\SC(\eps,\conf)$,
is defined as the smallest $m \in \nats \cup \{0\}$ for which
there exists a learning algorithm $\genericalg$ such that,
for every possible data distribution $\Px$, $\forall \target \in \C$, 
denoting $\hat{h} = \genericalg( (\ProcX_{1:m}(\Px),\target(\ProcX_{1:m}(\Px))) )$,
\begin{equation*}
\P\left( \er_{\Px}\left( \hat{h}; \target \right) \leq \eps \right) \geq 1-\conf.
\end{equation*}
If no such $m$ exists, define $\SC(\eps,\conf) = \infty$.
\end{definition}

The sample complexity is our primary object of study in this work.
We require a few additional definitions before proceeding.
Throughout, we use a natural extension of set notation to sequences:
for any finite sequences $\{a_{i}\}_{i=1}^{k}, \{b_{i}\}_{i=1}^{k^{\prime}}$,
we denote by $\{a_{i}\}_{i=1}^{k} \cup \{b_{i}\}_{i=1}^{k^{\prime}}$
the concatenated sequence $\{a_{1},\ldots,a_{k},b_{1},\ldots,b_{k^{\prime}}\}$.
For any set $A$, we denote by $\{a_{i}\}_{i=1}^{k} \cap A$
the subsequence comprised of all $a_{i}$ for which $a_{i} \in A$. 
Additionally, we write $b \in \{b_{i}\}_{i=1}^{k^{\prime}}$ to indicate $\exists i \leq k^{\prime}$ s.t. $b_{i} = b$,
and we write $\{a_{i}\}_{i=1}^{k} \subseteq \{b_{i}\}_{i=1}^{k^{\prime}}$
or $\{b_{i}\}_{i=1}^{k^{\prime}} \supseteq \{a_{i}\}_{i=1}^{k}$
to express that $a_{j} \in \{b_{i}\}_{i=1}^{k^{\prime}}$ for every $j \leq k$.
We also denote $| \{a_{i}\}_{i=1}^{k} | = k$ (the length of the sequence).
For any $k \in \nats \cup \{0\}$ and any sequence $S = \{(x_{1},y_{1}),\ldots,(x_{k},y_{k})\}$ of points in $\X \times \Y$, 
denote $\C[S] = \{ h \in \C : \forall (x,y) \in S, h(x) = y \}$,
referred to as the set of classifiers \emph{consistent} with $S$.

Following \citet*{vapnik:71}, we say a sequence $\{x_{1},\ldots,x_{k}\}$ of points in $\X$
is \emph{shattered} by $\C$ if $\forall y_{1},\ldots,y_{k} \in \Y$, $\exists h \in \C$
such that $\forall i \in \{1,\ldots,k\}$, $h(x_{i}) = y_{i}$: that is, there are $2^{k}$
distinct classifications of $\{x_{1},\ldots,x_{k}\}$ realized by classifiers in $\C$.
The Vapnik-Chervonenkis dimension (or \emph{VC dimension})
of $\C$ is then defined as the largest integer $k$ for which 
there exists a sequence $\{x_{1},\ldots,x_{k}\}$ in $\X$ shattered by $\C$;
if no such largest $k$ exists, the VC dimension is said to be infinite.
We denote by $\dim$ the VC dimension of $\C$.  This quantity is 
of fundamental importance in characterizing the sample complexity of 
PAC learning.  In particular, it is well known that the sample complexity 
is finite for any $\eps,\conf \in (0,1)$ if and only if $\dim < \infty$ 
\citep*{vapnik:82,blumer:89,ehrenfeucht:89}.
For simplicity of notation, for the remainder of this article
we suppose $\dim < \infty$; furthermore, note that our assumption of $|\C| \geq 3$ 
implies $\dim \geq 1$.

We adopt a common variation on big-O asymptotic notation, used in much of the 
learning theory literature.  Specifically, for functions $f,g : (0,1)^{2} \to [0,\infty)$, 
we let $f(\eps,\conf) = O( g(\eps,\conf) )$ denote the assertion that
$\exists \eps_{0},\conf_{0} \in (0,1)$ and $c_{0} \in (0,\infty)$ such that, 
$\forall \eps \in (0,\eps_{0})$, $\forall \conf \in (0,\conf_{0})$, 
$f(\eps,\conf) \leq c_{0} g(\eps,\conf)$; however, we also require that
the values $\eps_{0},\conf_{0},c_{0}$ in this definition be 
\emph{numerical constants}, meaning that they are 
\emph{independent of $\C$ and $\X$}.
For instance, this means $c_{0}$ cannot depend on $\dim$.
We equivalently write $f(\eps,\conf) = \Omega( g(\eps,\conf) )$ to assert 
that $g(\eps,\conf) = O( f(\eps,\conf) )$.  Finally, we write 
$f(\eps,\conf) = \Theta( g(\eps,\conf) )$ to assert 
that both $f(\eps,\conf) = O( g(\eps,\conf) )$ and $f(\eps,\conf) = \Omega( g(\eps,\conf) )$ hold.
We also sometimes write $O( g(\eps,\conf) )$ in an expression, 
as a place-holder for some function $f(\eps,\conf)$ satisfying $f(\eps,\conf) = O(g(\eps,\conf))$:
for instance, the statement $N(\eps,\conf) \leq \dim + O( \log(1/\conf) )$ expresses that
$\exists f(\eps,\conf) = O(\log(1/\conf))$ for which $N(\eps,\conf) \leq \dim + f(\eps,\conf)$.
Also, for any value $z \geq 0$, define $\Log(z) = \ln( \max\{z, e\} )$ and similarly $\Log_{2}(z) = \log_{2}(\max\{z,2\})$.

As is commonly required in the learning theory literature, we adopt the 
assumption that the events appearing in probability claims
below are indeed measurable.  For our purposes, this comes into effect only in the application
of classic generalization bounds for sample-consistent classifiers (Lemma~\ref{lem:vc-bound} below).
See \citet*{blumer:89} and \citet*{van-der-Vaart:96} for discussion of conditions on $\C$ sufficient for 
this measurability assumption to hold.

\section{Background}
\label{sec:background}

Our objective in this work is to establish \emph{sharp} sample complexity bounds.
As such, we should first review the known \emph{lower bounds} on $\SC(\eps,\conf)$.
A basic lower bound of $\frac{1-\eps}{\eps} \ln\left(\frac{1}{\conf}\right)$ 
was established by \citet*{blumer:89} for $0 < \eps < 1/2$ and $0 < \conf < 1$.
A second lower bound of $\frac{\dim-1}{32\eps}$ was supplied by \citet*{ehrenfeucht:89},
for $0 < \eps \leq 1/8$ and $0 < \conf \leq 1/100$.  
Taken together, these results imply that, for any $\eps \in (0,1/8]$ and $\conf \in (0,1/100]$, 
\begin{equation}
\label{eqn:lower-bound}
\SC(\eps,\conf) \geq \max\left\{ \frac{\dim-1}{32\eps}, \frac{1-\eps}{\eps}\ln\left(\frac{1}{\conf}\right) \right\} 
= \Omega\left( \frac{1}{\eps} \left( \dim + \Log\left(\frac{1}{\conf}\right) \right) \right).
\end{equation}

This lower bound is complemented by classic \emph{upper bounds} 
on the sample complexity.  In particular, \citet*{vapnik:82} and \citet*{blumer:89} 
established an upper bound of 
\begin{equation}
\label{eqn:classic-vc-upper-bound}
\SC(\eps,\conf) = O\left( \frac{1}{\eps} \left( \dim \Log\left(\frac{1}{\eps}\right) + \Log\left(\frac{1}{\conf}\right) \right) \right).
\end{equation}
They proved that this sample complexity bound is in fact achieved by any algorithm that returns a classifier $h \in \C[ (\ProcX_{1:m}(\Px),\target(\ProcX_{1:m}(\Px))) ]$,
also known as a \emph{sample-consistent learning algorithm} (or \emph{empirical risk minimization} algorithm).
A sometimes-better upper bound was established by \citet*{haussler:94}:
\begin{equation}
\label{eqn:hlw-upper-bound}
\SC(\eps,\conf) = O\left( \frac{\dim}{\eps} \Log\left(\frac{1}{\conf}\right) \right).
\end{equation}
This bound is achieved by a modified variant of the \emph{one-inclusion graph prediction algorithm},
a learning algorithm also proposed by \citet*{haussler:94}, which has been conjectured
to achieve the optimal sample complexity \citep*{warmuth:04}.

In very recent work, \citet*{simon:15} produced a breakthrough insight.
Specifically, by analyzing a learning algorithm based on a simple majority vote 
among classifiers consistent with distinct subsets of the training data, 
\citet*{simon:15} established that, 
for any $K \in \nats$, 
\begin{equation}
\label{eqn:simon-upper-bound}
\SC(\eps,\conf) = O\left( \frac{2^{2K}\sqrt{K}}{\eps} \left( \dim \log^{(K)}\left(\frac{1}{\eps}\right) + K + \Log\left(\frac{1}{\conf}\right) \right) \right),
\end{equation}
where $\log^{(K)}(x)$ is the $K$-times iterated logarithm: $\log^{(0)}(x) = \max\{x,1\}$ and $\log^{(K)}(x) = \max\{\log_{2}(\log^{(K-1)}(x)),1\}$.
In particular, one natural choice would be $K \approx \log^{*}\!\left(\frac{1}{\eps}\right)$,\footnote{The function 
$\log^{*}(x)$ is the iterated logarithm: the smallest $K \in \nats \cup \{0\}$ for which $\log^{(K)}(x) \leq 1$.  It is an extremely slowly growing function of $x$.}
which (one can show) optimizes the asymptotic dependence on $\eps$ in the bound,
yielding
\begin{equation*}
\SC(\eps,\conf) = O\left( \frac{1}{\eps} 2^{O(\log^{*}(1/\eps))} \left( \dim + \Log\left(\frac{1}{\conf}\right) \right) \right).
\end{equation*}
In general, the entire form of the bound \eqref{eqn:simon-upper-bound} is optimized (up to numerical constant factors) by choosing
$K = \max\left\{\log^{*}\!\left(\frac{1}{\eps}\right) - \log^{*}\!\left(\frac{1}{\dim}\Log\!\left(\frac{1}{\conf}\right)\right)+1, 1\right\}$.
Note that, with either of these choices of $K$, there is a range of $\eps$, $\conf$, and $\dim$ values for which the bound \eqref{eqn:simon-upper-bound} is strictly smaller than both 
\eqref{eqn:classic-vc-upper-bound} and \eqref{eqn:hlw-upper-bound}:
for instance, for small $\eps$, it suffices to have $\Log(1/\conf) \ll \dim \Log(1/\eps) / ( 2^{2\log^{*}(1/\eps)}\sqrt{\log^{*}(1/\eps)} )$
while $2^{2\log^{*}(1/\eps)}\sqrt{\log^{*}(1/\eps)} \ll \min\{\Log(1/\conf),\dim\}$.
However, this bound still does not quite match the form of the lower bound \eqref{eqn:lower-bound}.

There have also been many special-case analyses, studying restricted types
of concept spaces $\C$ for which the above gaps can be closed
\citep*[e.g.,][]{auer:07,darnstadt:15,hanneke:15b}.
However, these special conditions do not include many of the most-commonly
studied concept spaces, such as linear separators and multilayer neural networks.
There have also been a variety of studies that, in addition to restricting to specific 
concept spaces $\C$, also introduce strong restrictions on the data distribution $\Px$, 
and establish an upper bound of the same form as the lower bound \eqref{eqn:lower-bound}
under these restrictions \citep*[e.g.,][]{long:03,gine:06,long:09,hanneke:thesis,hanneke:15b,balcan:13}.
However, there are many interesting classes $\C$ and distributions $\Px$ for which 
these results do not imply any improvements over \eqref{eqn:classic-vc-upper-bound}.
Thus, in the present literature, there persists a gap between the lower bound \eqref{eqn:lower-bound}
and the minimum of all of the known upper bounds \eqref{eqn:classic-vc-upper-bound}, \eqref{eqn:hlw-upper-bound}, and \eqref{eqn:simon-upper-bound}
applicable to the \emph{general} case of an arbitrary concept space of a given VC dimension $\dim$
(under arbitrary data distributions).

In the present work, we establish a new upper bound for a novel learning algorithm, 
which holds for \emph{any} concept space $\C$, and which improves over all of the above general upper bounds 
in its joint dependence on $\eps$, $\conf$, and $\dim$.  In particular, it is \emph{optimal}, in the sense that it 
matches the lower bound \eqref{eqn:lower-bound} up to numerical constant factors.  This work thus 
solves a long-standing open problem, by determining the precise form of the optimal sample complexity,
up to numerical constant factors.

\section{Main Result}
\label{sec:main}

This section presents the main contributions of this work: 
a novel learning algorithm, and a proof that it achieves the 
optimal sample complexity.

\subsection{Sketch of the Approach}

The general approach used here builds on an argument of \citet*{simon:15},
which itself has roots in the analysis of sample-consistent learning 
algorithms by \citet*[][Section 2.9.1]{hanneke:thesis}.  The essential idea 
from \citet*{simon:15} is that,
if we have two classifiers, $\hat{h}$ and $\hat{g}$, 
the latter of which is an element of $\C$ consistent with an i.i.d.\! 
data set $\tilde{S}$ \emph{independent} from $\hat{h}$,
then we can analyze the probability that they \emph{both} make a mistake
on a random point by bounding the error rate of $\hat{h}$
under the distribution $\Px$, and bounding the error rate of $\hat{g}$
under the \emph{conditional} distribution given 
that $\hat{h}$ makes a mistake.  In particular, it will either
be the case that $\hat{h}$ itself has small error rate, or else
(if $\hat{h}$ has error rate larger than our desired bound)
with high probability, the number of points in $\tilde{S}$ 
contained in the error region of $\hat{h}$ will be at least some number $\propto \er_{\Px}(\hat{h};\target) |\tilde{S}|$; 
in the latter case, we can bound the conditional
error rate of $\hat{g}$ in terms of the number of such points via a classic generalization
bound for sample-consistent classifiers (Lemma~\ref{lem:vc-bound} below).
Multiplying this bound on the conditional error rate of $\hat{g}$
by the error rate of $\hat{h}$ results in a bound on the probability 
they both make a mistake.  More specifically, this argument yields a bound of the following form:
for an appropriate numerical constant $\tilde{c} \in (0,\infty)$, 
with probability at least $1-\conf$, $\forall \hat{g} \in \C[\tilde{S}]$, 
\begin{equation*}
\Px\left( x : \hat{h}(x) = \hat{g}(x) \neq \target(x) \right) 
\leq \frac{\tilde{c}}{|\tilde{S}|} \left( \dim \Log\left(\frac{\er_{\Px}(\hat{h};\target) |\tilde{S}|}{\dim}\right) + \Log\left(\frac{1}{\conf}\right) \right).
\end{equation*}

The original analysis of \citet*{simon:15} applied this reasoning repeatedly,
in an inductive argument, thereby bounding the probability that $K$ classifiers,
each consistent with one of $K$ independent training sets, all make a mistake on a
random point.  He then reasoned that the error rate of the majority
vote of $2K-1$ such classifiers can be bounded by the sum of these
bounds for all subsets of $K$ of these classifiers, since the majority
vote classifier agrees with at least $K$ of the constituent classifiers.

In the present work, we also consider a simple majority vote of 
a number of classifiers, but we alter the way the data is split up,
allowing significant overlaps among the subsamples.  In particular,
each classifier is trained on considerably more data this way.
We construct these subsamples recursively, motivated by an inductive analysis of the sample complexity.
At each stage, we have a working set $S$ of i.i.d. data points, and another sequence $T$ of data points,
referred to as the \emph{partially-constructed subsample}.
As a terminal case, if $|S|$ is smaller than a certain cutoff size, we generate a subsample $S \cup T$,
on which we will train a classifier $\hat{g} \in \C[S \cup T]$.
Otherwise (for the nonterminal case), we use (roughly) a constant fraction of the points in $S$ to form a subsequence $S_{0}$,
and make three recursive calls to the algorithm, using $S_{0}$ as the working set in each call.
By an inductive hypothesis, for each of these three recursive calls, with probability $1-\conf^{\prime}$, 
the majority vote of the classifiers trained on subsamples generated by that call
has error rate at most 
$\frac{c}{|S_{0}|} \left( \dim + \Log(1/\conf^{\prime}) \right)$, for an appropriate numerical constant $c$.
These three majority vote classifiers, denoted $h_{1}$, $h_{2}$, $h_{3}$, will each play the role of $\hat{h}$
in the argument above.

With the remaining constant fraction of data points in $S$ (i.e., those not used to form $S_{0}$),
we divide them into three independent subsequences $S_{1}$, $S_{2}$, $S_{3}$.
Then for each of the three recursive calls, we provide as its partially-constructed 
subsample (i.e., the ``$T$'' argument) a sequence $S_{i} \cup S_{j} \cup T$ with $i \neq j$;
specifically, for the $k^{{\rm th}}$ recursive call ($k \in \{1,2,3\}$), we take $\{i,j\} = \{1,2,3\} \setminus \{k\}$.
Since the argument $T$ is retained within the partially-constructed subsample passed to each recursive call,
a simple inductive argument reveals that, for each $i \in \{1,2,3\}$, $\forall k \in \{1,2,3\} \setminus \{i\}$, 
all of the classifiers $\hat{g}$ trained on subsamples generated in the $k^{{\rm th}}$ recursive 
call are contained in $\C[S_{i}]$.  Furthermore, since $S_{i}$ is not included in the argument to the $i^{{\rm th}}$
recursive call, $h_{i}$ and $S_{i}$ are independent.  Thus, by the argument discussed above, applied with 
$\hat{h} = h_{i}$ and $\tilde{S} = S_{i}$, we have that with probability at least $1-\conf^{\prime}$, 
for any $\hat{g}$ trained on a subsample generated in recursive calls $k \in \{1,2,3\} \setminus \{i\}$, 
the probability that \emph{both} $h_{i}$ \emph{and} $\hat{g}$ make a mistake on a random point is at most
$\frac{\tilde{c}}{|S_{i}|} \left( \dim \Log\left(\frac{\er_{\Px}(h_{i};\target) |S_{i}|}{\dim}\right) + \Log\left(\frac{1}{\conf^{\prime}}\right) \right)$.
Composing this with the aforementioned inductive hypothesis, recalling that $|S_{i}| \propto |S|$ and $|S_{0}| \propto |S|$, 
and simplifying by a bit of calculus,
this is at most
$\frac{c^{\prime}}{|S|} \left( \dim \Log( c ) + \Log\left(\frac{1}{\conf^{\prime}}\right) \right)$, for an appropriate numerical constant $c^{\prime}$.
By choosing $\conf^{\prime} \propto \conf$ appropriately, the union bound implies that, with probability at least $1-\conf$, 
this holds for all choices of $i \in \{1,2,3\}$.  Furthermore, by choosing $c$ sufficiently large, this bound is at most
$\frac{c}{12 |S|} \left( \dim + \Log\left(\frac{1}{\conf}\right) \right)$.

To complete the inductive argument, we then note that on any point $x$, the majority vote of all
of the classifiers (from all three recursive calls) must agree with at least one of the three classifiers
$h_{i}$, and must agree with at least $1/4$ of the classifiers $\hat{g}$ trained on subsamples
generated in recursive calls $k \in \{1,2,3\} \setminus \{i\}$.
Therefore, on any point $x$ for which 
the majority vote makes a mistake,
with probability at least $1/12$, a uniform random choice
of $i \in \{1,2,3\}$, and of $\hat{g}$ from recursive calls $k \in \{1,2,3\} \setminus \{i\}$,
results in $h_{i}$ and $\hat{g}$ that both make a mistake on $x$.
Applying this fact to a \emph{random} point $X \sim \Px$ (and invoking Fubini's theorem), 
this implies that the error rate of the majority
vote is at most $12$ times the average (over choices of $i$ and $\hat{g}$)
of the probabilities that $h_{i}$ and $\hat{g}$ both make a mistake on $X$.
Combined with the above bound, this is at most $\frac{c}{|S|} \left( \dim + \Log\left(\frac{1}{\conf}\right) \right)$.
The formal details are provided below.

\subsection{Formal Details}

For any $k \in \nats \cup \{0\}$, and any $S \in (\X \times \Y)^{k}$ with $\C[S] \neq \emptyset$, 
let $L(S)$ denote an arbitrary classifier $h$ in $\C[S]$, entirely determined by $S$:
that is, $L(\cdot)$ is a fixed sample-consistent learning algorithm (i.e., empirical risk minimizer).
For any $k \in \nats$ and sequence of data sets $\{S_{1},\ldots,S_{k}\}$,
denote $L(\{S_{1},\ldots,S_{k}\}) = \{L(S_{1}),\ldots,L(S_{k})\}$.
Also, for any values $y_{1},\ldots,y_{k} \in \Y$, 
define the majority function: $\Maj(y_{1},\ldots,y_{k}) = \sign\left( \sum_{i=1}^{k} y_{i} \right) = 2 \ind\left[ \sum_{i=1}^{k} y_{i} \geq 0 \right] - 1$.
We also overload this notation, defining the \emph{majority classifier} $\Maj(h_{1},\ldots,h_{k})(x) = \Maj(h_{1}(x),\ldots,h_{k}(x))$,
for any classifiers $h_{1},\ldots,h_{k}$.

Now consider the following recursive algorithm, which takes as input two 
finite data sets, $S$ and $T$, satisfying $\C[S \cup T] \neq \emptyset$, and returns a \emph{finite sequence of data sets}
(referred to as \emph{subsamples} of $S \cup T$).
The classifier used to achieve the new sample complexity bound below is simply the majority vote of the classifiers 
obtained by applying $L$ to these subsamples.

\begin{bigboxit}
Algorithm: $\alg(S;T)$\\
0. If $|S| \leq 3$\\
1. \quad Return $\{ S \cup T \}$\\
2. Let $S_{0}$ denote the first $|S| - 3 \lfloor |S|/4 \rfloor$ elements of $S$, $S_{1}$ the next $\lfloor |S|/4 \rfloor$ elements, \\
{\hskip 5mm}$S_{2}$ the next $\lfloor |S|/4 \rfloor$ elements, and $S_{3}$ the remaining $\lfloor |S|/4 \rfloor$ elements after that\\
3. Return $\alg(S_{0}; S_{2} \cup S_{3} \cup T) \cup \alg(S_{0}; S_{1} \cup S_{3} \cup T) \cup \alg(S_{0}; S_{1} \cup S_{2} \cup T)$
\end{bigboxit}

\begin{theorem}
\label{thm:main}
\begin{equation*}
\SC(\eps,\conf) = O\left( \frac{1}{\eps} \left( \dim + \Log\left(\frac{1}{\conf}\right) \right) \right).
\end{equation*}
In particular, a sample complexity of the form expressed on the right hand side
is achieved by the algorithm that returns the classifier 
$\Maj(L(\alg(S;\emptyset)))$, given any data set $S$.
\end{theorem}

Combined with \eqref{eqn:lower-bound}, this immediately implies the following corollary.

\begin{corollary}
\label{cor:matching-bound}
\begin{equation*}
\SC(\eps,\conf) = \Theta\left( \frac{1}{\eps} \left( \dim + \Log\left(\frac{1}{\conf}\right) \right) \right).
\end{equation*}
\end{corollary}

The algorithm $\alg$ is expressed above as a recursive method for constructing a sequence of subsamples,
as this form is most suitable for the arguments in the proof below.  However,
it should be noted that one can equivalently describe these constructed subsamples \emph{directly}, 
as the selection of which data points should be included in which subsamples
can be expressed as a simple function of the indices.  To illustrate this, consider
the simplest case in which $S = \{ (x_{0},y_{0}),\ldots,(x_{m-1},y_{m-1}) \}$ with
$m = 4^{\ell}$ for some $\ell \in \nats$: that is, $|S|$ is a power of $4$.  
In this case, let $\{T_{0},\ldots,T_{n-1}\}$ denote the sequence of labeled data sets 
returned by $\alg(S;\emptyset)$, and note that since each recursive call reduces $|S|$
by a factor of $4$ while making $3$ recursive calls, we have $n = 3^{\ell}$.
First, note that $(x_{0},y_{0})$ is contained in \emph{every} subsample $T_{i}$.
For the rest, consider any $i \in \{1,\ldots,m-1\}$ and $j \in \{0,\ldots,n-1\}$, 
and let us express $i$ in its base-$4$ representation as $i = \sum_{t=0}^{\ell-1} i_{t} 4^{t}$, where each $i_{t} \in \{0,1,2,3\}$,
and express $j$ in its base-$3$ representation as $j = \sum_{t=0}^{\ell-1} j_{t} 3^{t}$, where each $j_{t} \in \{0,1,2\}$.
Then it holds that $(x_{i},y_{i}) \in T_{j}$ if and only if the largest $t \in \{0,\ldots,\ell-1\}$ with $i_{t} \neq 0$
satisfies $i_{t}-1 \neq j_{t}$.  
%
%
This kind of direct description of the subsamples is also possible when $|S|$ is not a power of $4$, 
though a bit more complicated to express.

\subsection{Proof of Theorem~\ref{thm:main}}
\label{sec:proof}

The following classic result will be needed in the proof.
A bound of this type is implied by a theorem of \citet*{vapnik:82}; 
the version stated here features slightly smaller constant factors, 
obtained by 
\citet*{blumer:89}.\footnote{Specifically, it follows by combining 
their Theorem A2.1 and Proposition A2.1, setting the resulting expression
equal to $\conf$ and solving for $\eps$.}

\begin{lemma}
\label{lem:vc-bound}
For any $\conf \in (0,1)$, $m \in \nats$, $\target \in \C$, and any probability measure $P$ over $\X$, 
letting $Z_{1},\ldots,Z_{m}$ be independent $P$-distributed random variables,
with probability at least $1-\conf$, every $h \in \C[ \{ (Z_{i}, \target(Z_{i})) \}_{i=1}^{m} ]$
satisfies
\begin{equation*}
\er_{P}(h;\target) \leq \frac{2}{m} \left( \dim \Log_{2}\left(\frac{2 e m}{\dim}\right) + \Log_{2}\left(\frac{2}{\conf}\right) \right).
\end{equation*}
\end{lemma}

We are now ready for the proof of Theorem~\ref{thm:main}.

\begin{proof}[of Theorem~\ref{thm:main}]
Fix any $\target \in \C$ and probability measure $\Px$ over $\X$,
and for brevity, denote $\Data_{1:m} = (\ProcX_{1:m}(\Px),\target(\ProcX_{1:m}(\Px)))$, for each $m \in \nats$.
Also, for any classifier $h$, define ${\rm ER}(h) = \{x \in \X : h(x) \neq \target(x)\}$.

We begin by noting that, for any finite sequences $S$ and $T$ of points in $\X \times \Y$,
a straightforward inductive argument reveals that all of the subsamples $\hat{S}$ 
in the sequence returned by $\alg(S;T)$ satisfy $\hat{S} \subseteq S \cup T$ (since no additional data points are ever introduced in any step).  
Thus, if $\target \in \C[S]$ and $\target \in \C[T]$, then $\target \in \C[S] \cap \C[T] = \C[S \cup T] \subseteq \C[\hat{S}]$,
so that $\C[\hat{S}] \neq \emptyset$.  In particular, this means that, in this case, each of these subsamples $\hat{S}$
is a valid input to $L(\cdot)$, and thus $L(\alg(S;T))$ is a well-defined sequence of classifiers.
Furthermore, since the recursive calls all have $T$ as a subsequence of their second arguments,
and the terminal case (i.e., Step 1) includes this second argument in the constructed subsample,
another straightforward inductive argument implies that every subsample $\hat{S}$
returned by $\alg(S;T)$ satisfies $\hat{S} \supseteq T$.  Thus, in the case that $\target \in \C[S]$ and $\target \in \C[T]$,
by definition of $L$, we also have that every classifier $h$ in the sequence $L(\alg(S;T))$ satisfies 
$h \in \C[T]$.

Fix a numerical constant $c = 1800$.
We will prove by induction that, for any $m^{\prime} \in \nats$,
for every $\conf^{\prime} \in (0,1)$, and every finite sequence $T^{\prime}$ of points in $\X \times \Y$ with $\target \in \C[T^{\prime}]$,
with probability at least $1-\conf^{\prime}$,
the classifier $\hat{h}_{m^{\prime},T^{\prime}} = \Maj\left(L\left(\alg( \Data_{1:m^{\prime}} ; T^{\prime} )\right) \right)$ satisfies
\begin{equation}
\label{eqn:inductive-hypothesis}
\er_{\Px}( \hat{h}_{m^{\prime},T^{\prime}} ; \target ) \leq \frac{c}{m^{\prime}+1} \left( \dim + \ln\left(\frac{18}{\conf^{\prime}}\right) \right).
\end{equation}

First, as a base case, consider any $m^{\prime} \in \nats$ with $m^{\prime} \leq c \ln(18e) - 1$.
In this case, fix any $\conf^{\prime} \in (0,1)$ and any sequence $T^{\prime}$ with $\target \in \C[T^{\prime}]$.
Also note that $\target \in \C[\Data_{1:m^{\prime}}]$.  Thus, as discussed above, $\hat{h}_{m^{\prime},T^{\prime}}$ is a well-defined classifier.
We then trivially have
\begin{equation*}
\er_{\Px}( \hat{h}_{m^{\prime},T^{\prime}} ; \target ) 
\leq 1
\leq \frac{c}{m^{\prime}+1} \left( 1 + \ln(18) \right)
< \frac{c}{m^{\prime}+1} \left( \dim + \ln\left(\frac{18}{\conf^{\prime}}\right) \right),
\end{equation*}
so that \eqref{eqn:inductive-hypothesis} holds.

Now take as an inductive hypothesis that, for some $m \in \nats$ with $m > c \ln(18e) - 1$,
for every $m^{\prime} \in \nats$ with $m^{\prime} < m$,
we have that for every $\conf^{\prime} \in (0,1)$ and every finite sequence $T^{\prime}$ in $\X \times \Y$ with $\target \in \C[T^{\prime}]$,
with probability at least $1-\conf^{\prime}$, \eqref{eqn:inductive-hypothesis} is satisfied.
To complete the inductive proof, we aim to establish that this remains the case with $m^{\prime} = m$ as well.
Fix any $\conf \in (0,1)$ and any finite sequence $T$ of points in $\X \times \Y$ with $\target \in \C[T]$.  
Note that $c \ln(18e) - 1 \geq 3$, so that (since $|\Data_{1:m}| = m > c \ln(18e)-1$) 
we have $|\Data_{1:m}| \geq 4$, and hence the execution of $\alg(\Data_{1:m} ; T)$ returns in Step 3 (not Step 1).
Let $S_{0},S_{1},S_{2},S_{3}$ be as in the definition of $\alg(S;T)$,
with $S = \Data_{1:m}$.
Also denote $T_{1} = S_{2} \cup S_{3} \cup T$, $T_{2} = S_{1} \cup S_{3} \cup T$, $T_{3} = S_{1} \cup S_{2} \cup T$,
and for each $i \in \{1,2,3\}$, denote $\hmaji = \Maj\left( L\left( \alg( S_{0}; T_{i} ) \right) \right)$,
corresponding to the majority votes of classifiers trained on the subsamples from each of the three recursive calls in the algorithm.

Note that $S_{0} = \Data_{1:(m-3\lfloor m/4 \rfloor)}$.
Furthermore, since $m \geq 4$,
we have $1 \leq m-3\lfloor m/4 \rfloor < m$.
Also note that $\target \in \C[S_{i}]$ for each $i \in \{1,2,3\}$, which, together with 
the fact that $\target \in \C[T]$, implies $\target \in \C[T] \cap \bigcap_{j \in \{1,2,3\} \setminus \{i\}} \C[S_{j}] = \C[T_{i}]$ for each $i \in \{1,2,3\}$.
Thus, since $\target \in \C[S_{0}]$ as well, for each $i \in \{1,2,3\}$, $L\left( \alg( S_{0}; T_{i} ) \right)$ is a well-defined sequence of classifiers (as discussed above),
so that $\hmaji$ is also well-defined.  In particular, note that $\hmaji = \hat{h}_{(m-3\lfloor m/4 \rfloor),T_{i}}$.
Therefore, by the inductive hypothesis (applied under the conditional distribution given $S_{1}$, $S_{2}$, $S_{3}$, which are independent of $S_{0}$), 
combined with the law of total probability, for each $i \in \{1,2,3\}$, 
there is an event $E_{i}$ of probability at least $1-\conf/9$, on which 
\begin{equation}
\label{eqn:PRi}
\Px({\rm ER}(\hmaji)) \leq \frac{c}{|S_{0}|+1} \left( \dim + \ln\left(\frac{9 \cdot 18}{\conf}\right) \right)
\leq \frac{4c}{m} \left( \dim + \ln\left(\frac{9 \cdot 18}{\conf}\right) \right).
\end{equation}

Next, fix any $i \in \{1,2,3\}$,
and denote by $\left\{\left(Z_{i,1},\target(Z_{i,1})\right),\ldots,\left(Z_{i,N_{i}},\target(Z_{i,N_{i}})\right)\right\} = S_{i} \cap \left( {\rm ER}(\hmaji) \times \Y \right)$,
where $N_{i} = \left| S_{i} \cap \left( {\rm ER}(\hmaji) \times \Y \right) \right|$:
that is, $\left\{(Z_{i,t},\target(Z_{i,t}))\right\}_{t=1}^{N_{i}}$ is the subsequence of elements $(x,y)$ in $S_{i}$ for which $x \in {\rm ER}(\hmaji)$.
Note that, since $\hmaji$ and $S_{i}$ are independent,
$Z_{i,1},\ldots,Z_{i,N_{i}}$ are conditionally independent given $\hmaji$ and $N_{i}$,
each with conditional distribution $\Px( \cdot | {\rm ER}(\hmaji) )$ (if $N_{i} > 0$).
Thus, applying Lemma~\ref{lem:vc-bound} under the conditional distribution given $\hmaji$ and $N_{i}$,
combined with the law of total probability, we have that on an event $E_{i}^{\prime}$ of probability at least $1-\conf/9$,
if $N_{i} > 0$, then every $h \in \C\left[ \left\{(Z_{i,t},\target(Z_{i,t}))\right\}_{t=1}^{N_{i}} \right]$ satisfies
\begin{equation*}
\er_{\Px(\cdot|{\rm ER}(\hmaji))}(h;\target) \leq \frac{2}{N_{i}} \left( \dim \Log_{2}\left(\frac{2 e N_{i}}{\dim}\right) + \Log_{2}\left(\frac{18}{\conf}\right) \right).
\end{equation*}
Furthermore, as discussed above, each $j \in \{1,2,3\} \setminus \{i\}$
and $h \in L\left( \alg(S_{0} ; T_{j} ) \right)$ have $h \in \C[ T_{j} ]$,
and $T_{j} \supseteq S_{i} \supseteq \{(Z_{i,t},\target(Z_{i,t}))\}_{t=1}^{N_{i}}$, so that $\C[T_{j}] \subseteq \C\left[\left\{(Z_{i,t},\target(Z_{i,t}))\right\}_{t=1}^{N_{i}}\right]$.
It follows that every $h \in \bigcup_{j \in \{1,2,3\} \setminus \{i\}} L\left( \alg(S_{0} ; T_{j} ) \right)$ has $h \in \C\left[ \left\{(Z_{i,t},\target(Z_{i,t}))\right\}_{t=1}^{N_{i}} \right]$.
Thus, on the event $E_{i}^{\prime}$, if $N_{i} > 0$, $\forall h \in \bigcup_{j \in \{1,2,3\} \setminus \{i\}} L(\alg(S_{0}; T_{j}))$,
\begin{align}
& \Px( {\rm ER}(\hmaji) \cap {\rm ER}(h) ) = \Px( {\rm ER}(\hmaji) ) \Px\!\left( {\rm ER}(h) \middle| {\rm ER}(\hmaji) \right) \notag
\\ & = \Px( {\rm ER}(\hmaji) ) \er_{\Px(\cdot|{\rm ER}(\hmaji))}(h ; \target)
\leq \Px( {\rm ER}(\hmaji) ) \frac{2}{N_{i}} \left( \dim \Log_{2}\left(\frac{2 e N_{i}}{\dim}\right) + \Log_{2}\left(\frac{18}{\conf}\right) \right). \label{eqn:PRiRj}
\end{align}
Additionally, since $\hmaji$ and $S_{i}$ are independent,
by a Chernoff bound (applied under the conditional distribution given $\hmaji$) and the law of total probability, 
there is an event $E_{i}^{\prime\prime}$ of probability at least 
$1-\conf/9$,
on which, 
if $\Px({\rm ER}(\hmaji)) \geq \frac{23}{\lfloor m/4 \rfloor} \ln\left(\frac{9}{\conf}\right) \geq \frac{2 (10/3)^{2}}{\lfloor m/4 \rfloor} \ln\left(\frac{9}{\conf}\right)$,
then 
\begin{equation*}
N_{i} \geq (7/10) \Px({\rm ER}(\hmaji)) |S_{i}|
= (7/10) \Px({\rm ER}(\hmaji)) \lfloor m/4 \rfloor.
\end{equation*}
In particular, on $E_{i}^{\prime\prime}$,
if $\Px({\rm ER}(\hmaji)) \geq \frac{23}{\lfloor m/4 \rfloor} \ln\left(\frac{9}{\conf}\right)$, 
then the above inequality implies 
$N_{i} > 0$.

Combining this with \eqref{eqn:PRi} and \eqref{eqn:PRiRj}, 
and noting that $\Log_{2}(x) \leq \Log(x)/\ln(2)$ and $x \mapsto \frac{1}{x}\Log(c^{\prime} x)$ is nonincreasing on $(0,\infty)$ (for any fixed $c^{\prime} > 0$),
we have that on $E_{i} \cap E_{i}^{\prime} \cap E_{i}^{\prime\prime}$, 
if $\Px( {\rm ER}(\hmaji) ) \geq \frac{23}{\lfloor m/4 \rfloor} \ln\left(\frac{9}{\conf}\right)$, 
then every $h \in \bigcup_{j \in \{1,2,3\} \setminus \{i\}} L(\alg(S_{0}; T_{j}))$ has
\begin{align}
\Px({\rm ER}(\hmaji) \cap {\rm ER}(h) ) & \leq \frac{20}{7\ln(2)\lfloor m/4 \rfloor} \left( \dim \Log\left(\frac{2 e (7/10) \Px({\rm ER}(\hmaji)) \lfloor m/4 \rfloor}{\dim}\right) + \Log\left(\frac{18}{\conf}\right) \right) \notag
\\ & \leq \frac{20}{7\ln(2)\lfloor m/4 \rfloor} \left( \dim \Log\left(\frac{(7/5) e c \left( \dim + \ln\left(\frac{9 \cdot 18}{\conf}\right) \right)}{\dim}\right) + \Log\left(\frac{18}{\conf}\right) \right) \notag
\\ & \leq \frac{20}{7\ln(2)\lfloor m/4 \rfloor} \left( \dim \Log\!\left(  (2/5) c \left( (7/2)e + \frac{7e}{\dim} \!\ln\!\left(\frac{18}{\conf}\right) \right) \right) + \Log\!\left(\frac{18}{\conf}\right) \right) \notag
\\ & \leq \frac{20}{7\ln(2)\lfloor m/4 \rfloor} \left( \dim \ln\!\left( (9/5) e c \right) + 8 \ln\left(\frac{18}{\conf}\right) \right), \label{eqn:messy-log}
\end{align}
where this last inequality is due to Lemma~\ref{lem:log-factor} in Appendix~\ref{app:lemmas}.
Since $m > c \ln(18e) - 1 > 3200$,
we have 
$\lfloor m/4 \rfloor > (m-4)/4 > \frac{799}{800} \frac{m}{4} > \frac{799}{800} \frac{3200}{3201} \frac{m+1}{4}$.
Plugging this relaxation into the above bound, combined with numerical calculation of the logarithmic factor (with $c$ as defined above),
we find that the expression in \eqref{eqn:messy-log} is less than
\begin{equation*}
\frac{150}{m+1} \left( \dim + \ln\left(\frac{18}{\conf}\right) \right).
\end{equation*}
Additionally, if $\Px({\rm ER}(\hmaji)) < \frac{23}{\lfloor m/4 \rfloor} \ln\left(\frac{9}{\conf}\right)$,
then monotonicity of measures implies
\begin{equation*}
\Px( {\rm ER}(\hmaji) \cap {\rm ER}(h) ) 
\leq \Px( {\rm ER}(\hmaji) ) < \frac{23}{\lfloor m/4 \rfloor} \ln\left(\frac{9}{\conf}\right)
< \frac{150}{m+1} \left( \dim + \ln\left(\frac{18}{\conf}\right) \right),
\end{equation*}
again using the above lower bound on $\lfloor m/4 \rfloor$ for this last inequality.
Thus, regardless of the value of $\Px( {\rm ER}(\hmaji) )$, on the event $E_{i} \cap E_{i}^{\prime} \cap E_{i}^{\prime\prime}$,
we have $\forall h \in \bigcup_{j \in \{1,2,3\} \setminus \{i\}} L(\alg(S_{0}; T_{j}))$,
\begin{equation*}
\Px( {\rm ER}(\hmaji) \cap {\rm ER}(h) ) < \frac{150}{m+1} \left( \dim + \ln\left(\frac{18}{\conf}\right) \right).
\end{equation*}

Now denote $h_{{\rm maj}} = \hat{h}_{m,T} = \Maj\left( L(\alg(S;T)) \right)$, again with $S = \Data_{1:m}$.
By definition of $\Maj(\cdot)$, for any $x \in \X$, at least $1/2$ of the classifiers $h$ in the sequence
$L(\alg(S;T))$ have $h(x) = h_{{\rm maj}}(x)$.
From this fact, the strong form of the pigeonhole principle implies that at least one $i \in \{1,2,3\}$
has $\hmaji(x) = h_{{\rm maj}}(x)$ (i.e., the majority vote must agree with the majority of classifiers in at least one of the three subsequences of classifiers).  
Furthermore, since each $\alg(S_{0}; T_{j})$ (with $j \in \{1,2,3\}$) supplies an equal number of 
entries to the sequence $\alg(S;T)$ (by a straightforward inductive argument), 
for each $i \in \{1,2,3\}$, at least $1/4$ of the classifiers $h$ in $\cup_{j \in \{1,2,3\} \setminus \{i\}} L(\alg(S_{0}; T_{j}))$ have $h(x) = h_{{\rm maj}}(x)$:
that is, since $|L(\alg(S_{0};T_{i}))| = (1/3)|L(\alg(S;T))|$, we must have at least $(1/6)|L(\alg(S;T))| = (1/4)\left|\bigcup_{j \in \{1,2,3\} \setminus \{i\}} L(\alg(S_{0};T_{j}))\right|$
classifiers $h$ in $\bigcup_{j \in \{1,2,3\} \setminus \{i\}} L(S_{0};T_{j})$ with $h(x) = h_{{\rm maj}}(x)$ in order to meet the total of at least $(1/2)|L(\alg(S;T))|$ classifiers $h \in L(\alg(S;T))$ with $h(x) = h_{{\rm maj}}(x)$.
In particular, letting $I$ be a random variable uniformly distributed on $\{1,2,3\}$ (independent of the data),
and letting $\tilde{h}$ be a random variable conditionally (given $I$ and $S$) uniformly distributed on the classifiers 
$\cup_{j \in \{1,2,3\} \setminus \{I\}} L(\alg(S_{0} ; T_{j}))$,
this implies that for any fixed $x \in {\rm ER}(h_{{\rm maj}})$,
with conditional (given $S$) probability at least $1/12$, 
$\hmaj{I}(x) = \tilde{h}(x) = h_{{\rm maj}}(x)$, so that $x \in {\rm ER}( \hmaj{I} ) \cap {\rm ER}(\tilde{h})$ as well.
Thus, for a random variable $X \sim \Px$ (independent of the data and $I$,$\tilde{h}$), the law of total probability and monotonicity of conditional expectations imply
\begin{align*}
& \E\left[ \Px\left( {\rm ER}( \hmaj{I} ) \cap {\rm ER}(\tilde{h}) \right) \middle| S \right]
= \E\left[ \P\left( X \in {\rm ER}( \hmaj{I} ) \cap {\rm ER}(\tilde{h}) \middle| I,\tilde{h},S\right) \middle| S \right]
\\ & = \E\left[ \ind\left[ X \in {\rm ER}( \hmaj{I} ) \cap {\rm ER}(\tilde{h}) \right] \middle| S \right]
= \E\left[ \P\left( X \in {\rm ER}( \hmaj{I} ) \cap {\rm ER}(\tilde{h}) \middle| S,X \right) \middle| S \right]
\\ & \geq \E\left[ \P\left( X \in {\rm ER}( \hmaj{I} ) \cap {\rm ER}(\tilde{h}) \middle| S,X \right) \ind\left[ X \in {\rm ER}( h_{{\rm maj}} ) \right] \middle| S \right]
\\ & \geq \E\left[ (1/12) \ind\left[ X \in {\rm ER}( h_{{\rm maj}} ) \right]\middle| S \right]
= (1/12) \er_{\Px}(h_{{\rm maj}};\target).
\end{align*}
Thus, on the event $\bigcap_{i \in \{1,2,3\}} E_{i} \cap E_{i}^{\prime} \cap E_{i}^{\prime\prime}$,
\begin{align*}
\er_{\Px}(h_{{\rm maj}};\target)
& \leq 12 \E\left[ \Px\left( {\rm ER}( \hmaj{I} ) \cap {\rm ER}(\tilde{h}) \right) \middle| S \right]
\\ & \leq 12 \max_{i \in \{1,2,3\}} \max_{j \in \{1,2,3\} \setminus \{i\}} \max_{h \in L(\alg(S_{0}; T_{j}))} \Px({\rm ER}(\hmaji) \cap {\rm ER}(h))
\\ & < \frac{1800}{m+1} \left( \dim + \ln\left(\frac{18}{\conf}\right) \right)
= \frac{c}{m+1} \left( \dim + \ln\left(\frac{18}{\conf}\right) \right).
\end{align*}
Furthermore, by the union bound, the event $\bigcap_{i \in \{1,2,3\}} E_{i} \cap E_{i}^{\prime} \cap E_{i}^{\prime\prime}$ has probability at least $1-\conf$.
Thus, since this argument holds for any $\conf \in (0,1)$ and any finite sequence $T$ with $\target \in \C[T]$,
we have succeeded in extending the inductive hypothesis to include $m^{\prime} = m$.

By the principle of induction, we have established the claim that, 
$\forall m \!\in\! \nats$,
$\forall \conf \!\in\! (0,1)$,
for every finite sequence $T$ of points in $\X \times \Y$ with $\target \in \C[T]$,
with probability at least $1-\conf$, 
\begin{equation}
\label{eqn:the-error-bound}
\er_{\Px}( \hat{h}_{m,T} ; \target ) \leq \frac{c}{m+1} \left( \dim + \ln\left(\frac{18}{\conf}\right) \right).
\end{equation}
To complete the proof, we simply take $T = \emptyset$ (the empty sequence),
and note that, for any $\eps,\conf \in (0,1)$, for any value $m \in \nats$ of size at least
\begin{equation}
\label{eqn:the-sample-complexity-bound}
\left\lfloor \frac{c}{\eps} \left( \dim + \ln\left(\frac{18}{\conf}\right) \right) \right\rfloor,
\end{equation}
the right hand side of \eqref{eqn:the-error-bound} is less than $\eps$,
so that $\Maj(L(\alg(\cdot;\emptyset)))$ achieves a sample complexity equal the expression in \eqref{eqn:the-sample-complexity-bound}.
In particular, this implies
\begin{equation*}
\SC(\eps,\conf) \leq \frac{c}{\eps} \left( \dim + \ln\left(\frac{18}{\conf}\right) \right)
= O\left( \frac{1}{\eps} \left( \dim + \Log\left(\frac{1}{\conf}\right) \right) \right).
\end{equation*}
\end{proof}

\section{Remarks}
\label{sec:remarks}

On the issue of computational complexity, we note that the construction of subsamples by
$\alg$ can be quite efficient.  Since the branching factor is $3$, while $|S_{0}|$ is reduced
by roughly a factor of $4$ with each recursive call, the total number of subsamples 
returned by $\alg(S;\emptyset)$ is a sublinear function of $|S|$.
Furthermore, with appropriate data structures, 
the operations within each node of the recursion tree can be performed in constant time.
Indeed, as discussed above, 
one can directly determine which data points to include in each subsample via a simple function of the indices,
so that construction of these subsamples truly is computationally easy.

The only remaining significant computational issue in the learning algorithm is then the
efficiency of the sample-consistent base learner $L$. 
The existence of such an algorithm $L$, with running time polynomial in the size of the input sequence and $\dim$,
has been the subject of much investigation for a variety of concept spaces $\C$ \citep*[e.g.,][]{khachiyan:79,karmarkar:84b,valiant:84,pitt:88,helmbold:90}.
For instance, the commonly-used concept space of \emph{linear separators} admits such an algorithm
(where $L(S)$ may be expressed as a solution of a system of linear inequalities).
One can easily extend Theorem~\ref{thm:main} to admit base learners $L$ 
that are \emph{improper} (i.e., which may return classifiers not contained in $\C$), 
as long as they are guaranteed to return a sample-consistent classifier 
in \emph{some} hypothesis space $\H$ of VC dimension $O(\dim)$. 
Furthermore, as discussed by \citet*{pitt:88} and \citet*{haussler:91},
there is a simple technique for efficiently converting \emph{any} 
efficient PAC learning algorithm for $\C$, returning classifiers in $\H$,
into an efficient algorithm $L$ for finding (with probability $1-\conf^{\prime}$) a classifier 
in $\H$ consistent with a given data set $S$ with $\C[S] \neq \emptyset$.
%
%
Additionally, though the analysis above takes $L$ to be deterministic, this merely serves to 
simplify the notation in the proof, 
and it is straightforward to generalize the proof to allow randomized 
base learners $L$, including those that fail to return a sample-consistent classifier with 
some probability $\conf^{\prime}$ taken sufficiently small 
(e.g., $\conf^{\prime} = \conf/(2|\alg(S;\emptyset)|)$).
Composing these facts, we may conclude that, 
for any concept space $\C$ that is efficiently PAC learnable 
using a hypothesis space $\H$ of VC dimension $O(\dim)$,
there exists an efficient PAC learning algorithm for $\C$ with 
optimal sample complexity (up to numerical constant factors).

We conclude by noting that the constant factors obtained in the above proof are 
quite large.  Some small refinements are possible within the current approach: 
for instance, by choosing the $S_{i}$ subsequences slightly larger (e.g., $(3/10)|S|$),
or using a tighter form of the Chernoff bound when lower-bounding $N_{i}$.
However, there are inherent limitations to the approach used here, 
so that reducing the constant factors by more than, say, one order of magnitude,
may require significant changes to some part of the analysis, and perhaps the algorithm itself.
For this reason, it seems the next step in the study of $\SC(\eps,\conf)$ should 
be to search for strategies yielding refined constant factors.
%
In particular, \citet*{warmuth:04} has conjectured that the one-inclusion graph prediction algorithm
also achieves a sample complexity of the optimal form.  This conjecture remains
open at this time.  The one-inclusion graph predictor is known to achieve
the optimal sample complexity in the closely-related \emph{prediction model} of learning 
(where the objective is to achieve \emph{expected} error rate at most $\eps$),
with a numerical constant factor very close to optimal \citep*{haussler:94}.
It therefore seems likely that a (positive) resolution of Warmuth's one-inclusion graph conjecture may also lead to 
improvements in constant factors compared to the bound on $\SC(\eps,\conf)$ established in the present work.

\appendix

\section{A Technical Lemma}
\label{app:lemmas}

The following basic lemma is useful in the proof of Theorem~\ref{thm:main}.\footnote{This lemma and proof also appear in a sibling paper \citep*{hanneke:15b}.}

\begin{lemma}
\label{lem:log-factor}
For any $\a,\b,c_{1} \in [1,\infty)$ and $c_{2} \in [0,\infty)$, 
\begin{equation*}
\a \ln\left( c_{1} \left( c_{2} + \frac{\b}{\a} \right) \right)
\leq \a \ln\left( c_{1} (c_{2} + e) \right) + \frac{1}{e} \b.
\end{equation*}
\end{lemma}
\begin{proof}
If $\frac{\b}{\a} \leq e$, then monotonicity of $\ln(\cdot)$ implies 
\begin{equation*}
\a \ln\left( c_{1} \left( c_{2} + \frac{\b}{\a} \right) \right)
\leq \a \ln( c_{1} ( c_{2} + e ) )
\leq \a \ln( c_{1} (c_{2} + e) ) + \frac{1}{e} \b.
\end{equation*}

On the other hand, if $\frac{\b}{\a} > e$, then 
\begin{equation*}
\a \ln\left( c_{1} \left( c_{2} + \frac{\b}{\a} \right) \right)
\leq \a \ln\left( c_{1} \max\{ c_{2}, 2 \} \frac{\b}{\a} \right)
= \a \ln\left( c_{1} \max\{ c_{2}, 2 \} \right) + \a \ln\left( \frac{\b}{\a} \right).
\end{equation*}
The first term in the rightmost expression is at most $\a \ln( c_{1} (c_{2} + 2) ) \leq \a \ln( c_{1} (c_{2} + e) )$.
The second term in the rightmost expression can be rewritten as $\b \frac{\ln( \b/\a )}{\b / \a}$.
Since $x \mapsto \ln(x)/x$ is nonincreasing on $(e,\infty)$, in the case $\frac{\b}{\a} > e$, this is at most $\frac{1}{e} \b$.
Together, we have that 
\begin{equation*}
\a \ln\left( c_{1} \left( c_{2} + \frac{\b}{\a} \right) \right)
\leq \a \ln( c_{1} (c_{2} + e) ) + \frac{1}{e} \b
\end{equation*}
in this case as well.
\end{proof}

\acks{I would like to express my sincere thanks to Hans Simon and Amit Daniely
for helpful comments on a preliminary attempt at a solution.}

\bibliography{learning}

\begin{thebibliography}{22}
\providecommand{\natexlab}[1]{#1}
\providecommand{\url}[1]{\texttt{#1}}
\expandafter\ifx\csname urlstyle\endcsname\relax
  \providecommand{\doi}[1]{doi: #1}\else
  \providecommand{\doi}{doi: \begingroup \urlstyle{rm}\Url}\fi

\bibitem[Auer and Ortner(2007)]{auer:07}
P.~Auer and R.~Ortner.
\newblock A new {PAC} bound for intersection-closed concept classes.
\newblock \emph{Machine Learning}, 66\penalty0 (2-3):\penalty0 151--163, 2007.

\bibitem[Balcan and Long(2013)]{balcan:13}
M.-F. Balcan and P.~M. Long.
\newblock Active and passive learning of linear separators under log-concave
  distributions.
\newblock In \emph{Proceedings of the $26^{{\rm th}}$ Conference on Learning
  Theory}, 2013.

\bibitem[Blumer et~al.(1989)Blumer, Ehrenfeucht, Haussler, and
  Warmuth]{blumer:89}
A.~Blumer, A.~Ehrenfeucht, D.~Haussler, and M.~K. Warmuth.
\newblock Learnability and the {Vapnik-Chervonenkis} dimension.
\newblock \emph{Journal of the Association for Computing Machinery},
  36\penalty0 (4):\penalty0 929--965, 1989.

\bibitem[Bshouty et~al.(2009)Bshouty, Li, and Long]{long:09}
N.~H. Bshouty, Y.~Li, and P.~M. Long.
\newblock Using the doubling dimension to analyze the generalization of
  learning algorithms.
\newblock \emph{Journal of Computer and System Sciences}, 75\penalty0
  (6):\penalty0 323--335, 2009.

\bibitem[Darnst\"{a}dt(2015)]{darnstadt:15}
M.~Darnst\"{a}dt.
\newblock The optimal {PAC} bound for intersection-closed concept classes.
\newblock \emph{Information Processing Letters}, 115\penalty0 (4):\penalty0
  458--461, 2015.

\bibitem[Ehrenfeucht et~al.(1989)Ehrenfeucht, Haussler, Kearns, and
  Valiant]{ehrenfeucht:89}
A.~Ehrenfeucht, D.~Haussler, M.~Kearns, and L.~G. Valiant.
\newblock A general lower bound on the number of examples needed for learning.
\newblock \emph{Information and Computation}, 82:\penalty0 247--261, 1989.

\bibitem[Gin\'{e} and Koltchinskii(2006)]{gine:06}
E.~Gin\'{e} and V.~Koltchinskii.
\newblock Concentration inequalities and asymptotic results for ratio type
  empirical processes.
\newblock \emph{The Annals of Probability}, 34\penalty0 (3):\penalty0
  1143--1216, 2006.

\bibitem[Hanneke(2009)]{hanneke:thesis}
S.~Hanneke.
\newblock \emph{Theoretical Foundations of Active Learning}.
\newblock PhD thesis, Machine Learning Department, School of Computer Science,
  Carnegie Mellon University, 2009.

\bibitem[Hanneke(2015)]{hanneke:15b}
S.~Hanneke.
\newblock Refined error bounds for several learning algorithms.
\newblock \emph{arXiv:1512.07146}, 2015.

\bibitem[Haussler et~al.(1991)Haussler, Kearns, Littlestone, and
  Warmuth]{haussler:91}
D.~Haussler, M.~Kearns, N.~Littlestone, and M.~K. Warmuth.
\newblock Equivalence of models of polynomial learnability.
\newblock \emph{Information and Computation}, 95\penalty0 (2):\penalty0
  129--161, 1991.

\bibitem[Haussler et~al.(1994)Haussler, Littlestone, and Warmuth]{haussler:94}
D.~Haussler, N.~Littlestone, and M.~K. Warmuth.
\newblock Predicting $\{0,1\}$-functions on randomly drawn points.
\newblock \emph{Information and Computation}, 115:\penalty0 248--292, 1994.

\bibitem[Helmbold et~al.(1990)Helmbold, Sloan, and Warmuth]{helmbold:90}
D.~Helmbold, R.~Sloan, and M.~K. Warmuth.
\newblock Learning nested differences of intersection-closed concept classes.
\newblock \emph{Machine Learning}, 5\penalty0 (2):\penalty0 165--196, 1990.

\bibitem[Karmarkar(1984)]{karmarkar:84b}
N.~Karmarkar.
\newblock A new polynomial-time algorithm for linear programming.
\newblock \emph{Combinatorica}, 4\penalty0 (4):\penalty0 373--395, 1984.

\bibitem[Khachiyan(1979)]{khachiyan:79}
L.~G. Khachiyan.
\newblock A polynomial algorithm in linear programming.
\newblock \emph{Soviet Mathematics Doklady}, 20:\penalty0 191--194, 1979.

\bibitem[Long(2003)]{long:03}
P.~M. Long.
\newblock An upper bound on the sample complexity of {PAC} learning halfspaces
  with respect to the uniform distribution.
\newblock \emph{Information Processing Letters}, 87\penalty0 (5):\penalty0
  229--234, 2003.

\bibitem[Pitt and Valiant(1988)]{pitt:88}
L.~Pitt and L.~G. Valiant.
\newblock Computational limitations on learning from examples.
\newblock \emph{Journal of the Association for Computing Machinery},
  35\penalty0 (4):\penalty0 965--984, 1988.

\bibitem[Simon(2015)]{simon:15}
H.~Simon.
\newblock An almost optimal {PAC} algorithm.
\newblock In \emph{Proceedings of the $28^{{\rm th}}$ Conference on Learning
  Theory}, 2015.

\bibitem[Valiant(1984)]{valiant:84}
L.~G. Valiant.
\newblock A theory of the learnable.
\newblock \emph{Communications of the {ACM}}, 27\penalty0 (11):\penalty0
  1134--1142, 1984.

\bibitem[van~der Vaart and Wellner(1996)]{van-der-Vaart:96}
A.~W. van~der Vaart and J.~A. Wellner.
\newblock \emph{Weak Convergence and Empirical Processes}.
\newblock Springer, 1996.

\bibitem[Vapnik(1982)]{vapnik:82}
V.~Vapnik.
\newblock \emph{Estimation of Dependencies Based on Empirical Data}.
\newblock Springer-Verlag, New York, 1982.

\bibitem[Vapnik and Chervonenkis(1971)]{vapnik:71}
V.~Vapnik and A.~Chervonenkis.
\newblock On the uniform convergence of relative frequencies of events to their
  probabilities.
\newblock \emph{Theory of Probability and its Applications}, 16:\penalty0
  264--280, 1971.

\bibitem[Warmuth(2004)]{warmuth:04}
M.~K. Warmuth.
\newblock The optimal {PAC} algorithm.
\newblock In \emph{Proceedings of the $17^{{\rm th}}$ Conference on Learning
  Theory}, 2004.

\end{thebibliography}

\end{document}